\newif\iffollowingorders
\followingordersfalse
\newif\ifcomments
\commentstrue
\def\mypapertitle{The Curse of Concentration in Robust Learning:  \\ Evasion and Poisoning Attacks from Concentration of Measure}
\def\mypapershorttitle{The Curse of Concentration in Robust Learning:  \\ Unavoidable Evasion and Poisoning Attacks from Concentration}
\def\mypaperauthors{Saeed Mahloujifar\thanks{University of Virginia. Supported by University of Virginia's SEAS Research Innovation Awards.} \and Dimitrios I.~Diochnos\thanks{University of Virginia.} \and Mohammad Mahmoody\thanks{University of Virginia. Supported by NSF CAREER award CCF-1350939,   and two University of Virginia's SEAS Research Innovation Awards.}}

\iffollowingorders
\relax
\documentclass[letterpaper]{article} 
\usepackage{pdfpages}
\usepackage{aaai19}  
\usepackage{times}  
\usepackage{helvet}  
\usepackage{courier}  
\usepackage{url}  
\usepackage{graphicx}  
\else
\documentclass[letterpaper,11pt]{article}
\usepackage[margin=1.2in]{geometry}
\fi
\usepackage{times}

\usepackage{xfrac}

\usepackage{paralist} 
\usepackage{psfrag}

\iffollowingorders
\usepackage{xcolor}
\newcommand{\citep}[1]{\cite{#1}}
\else
\usepackage{paralist}
\usepackage[numbers,square,sort&compress]{natbib}

\usepackage[usenames]{color}
\usepackage[
   colorlinks%
   ,plainpages=false
   ,hypertexnames=true
   ,naturalnames
   ,hyperindex
   ,linkcolor=blue
   ,citecolor=blue
   ,urlcolor=RoyalBlue
   ,urlcolor=blue
   ,pdfauthor={}
   ,pdftitle={}
   ,pdfsubject={Adversarial Perturbation}
]{hyperref}
\usepackage[
  english 
  , ruled 
  , linesnumbered 
  , vlined 
]{algorithm2e}
\fi

\definecolor{RoyalBlue}{cmyk}{1, 0.50, 0, 0}
\definecolor{ForestGreen}{cmyk}{0.864, 0.0, 0.429, 0.396}
\definecolor{Brown}{cmyk}{0.0,0.692,0.925,0.529}

\usepackage{amssymb,amsmath,amsthm}
\usepackage[makeroom]{cancel}
\usepackage{thmtools,thm-restate}
\usepackage{comment}
\usepackage{multirow}

\usepackage{centernot}

\newcommand{\conf}{\mathsf{Conf}}
\newcommand{\charac}{\mathbf{1}}
\newcommand{\h}{\mathcal{H}}
\newcommand{\maxSE}{{m}}
\usepackage{upgreek}
\newcommand{\msr}{{\bm{\upmu}}}
\newcommand{\Leb}{{\bm{\uplambda}}}
\newcommand{\mssr}{\Leb}

\newcommand{\erf}{\mathsf{erf}}
\newcommand{\Exists}{\exists\,}
\newcommand{\Forall}{\forall\,}
\newcommand{\diam}{\mathsf{Diam}}
\usepackage[utf8]{inputenc}
\newcommand{\Levy}{Lévy\xspace}

\newcommand{\metric}{{\mathbf{d}}}
\newcommand{\Ball}{\mathcal{B}all}

\newcommand{\PC}{^{\mathrm{PC}}}

\newcommand{\HD}{\mathsf{HD}}

\newcommand{\Rob}{\mathsf{Rob}}
\newcommand{\comul}{{F}}

\newcommand{\xVec}{\ol{x}}

\newcommand{\Err}{\mathcal{E}}

\newcommand{\C}{\cC}

\newcommand{\poly}{\mathrm{poly}}

\newcommand{\Rplus}{\R_+}
\newcommand{\Risk}{\mathsf{Risk}}

\newcommand{\fhat}[2]{\ifthenelse{\equal{#2}{}}{\hat{f}[#1]}{\ifthenelse{\equal{#2}{0}}{\hat{f}[\emptyset]}{\hat{f}[#1_{\leq #2}]}}}
\newcommand{\gain}[2]{\ifthenelse{\equal{#2}{}}{g[#1]}{g[#1_{\leq #2}]}}

\newcommand{\pr}[2][]{\Pr_{\ifthenelse{\isempty{#1}}{}{{#1}}}\left[{#2}\right]}

\makeatother

\newcommand{\sm}{\setminus}

\newcommand{\e}{\mathrm{e}}

\usepackage{graphicx,accents}

\newcommand{\remove}[1]{}

\newcommand{\ol}{\overline}
\newcommand{\wt}[1]{\widetilde{#1}}

\newcommand{\se}{\subseteq}

\newcommand{\set}[1]{\left\{ #1 \right\}}
\newcommand{\con}{{\bm \upalpha}}

\newcommand{\bits}{\{0,1\}}

\newcommand{\To}{\mapsto}

\newcommand{\R}{{\mathbb R}}
\newcommand{\N}{{\mathbb N}}

\newcommand{\cC}{{\mathcal C}}

\newcommand{\cE}{{\mathcal E}}
\newcommand{\cF}{{\mathcal F}}

\newcommand{\cS}{{\mathcal S}}
\newcommand{\cT}{{\mathcal T}}

\newcommand{\cX}{{\mathcal X}}
\newcommand{\cY}{{\mathcal Y}}
\newcommand{\cZ}{{\mathcal Z}}

\usepackage{bm}

\newcommand{\eps}{\varepsilon}

\newcommand{\Exp}{\operatorname*{\mathbf{E}}}
\newcommand{\Ex}{\Exp}

\newcommand{\Supp}{\mathrm{Supp}}

\newtheorem{theorem}{Theorem}[section]
\newtheorem{claim}[theorem]{Claim}

\newtheorem{lemma}[theorem]{Lemma}
\newtheorem{corollary}[theorem]{Corollary}

\newtheorem{definition}[theorem]{Definition}

\newtheorem{remark}[theorem]{Remark}

\newcommand{\sdotfill}{\textcolor[rgb]{0.8,0.8,0.8}{\dotfill}} 

\newtheorem{proto}[theorem]{Protocol}
\newtheorem{protoc}[theorem]{Protocol}

\newcommand{\namedref}[2]{#1~\ref{#2}}

\newcommand{\torestate}[3]{%
\expandafter \def \csname BBRESTATE #2 \endcsname{#3}
\theoremstyle{plain}
\newtheorem{BBRESTATETHMNUM#2}[theorem]{#1}
\begin{BBRESTATETHMNUM#2}\label{#2}\csname BBRESTATE #2 \endcsname   \end{BBRESTATETHMNUM#2}
\newtheorem*{BBRESTATETHMNONNUM#2}{\namedref{#1}{#2}}
}

\newcommand{\restate}[1]{\begin{BBRESTATETHMNONNUM#1}[Restated] \csname BBRESTATE #1 \endcsname
\end{BBRESTATETHMNONNUM#1}}

\usepackage{xspace}
\newcommand{\X}{\ensuremath{\mathcal{X}}\xspace} 
\newcommand{\Y}{\ensuremath{\mathcal{Y}}\xspace} 
\renewcommand{\H}{\ensuremath{\mathcal H}\xspace} 

\newcommand{\expectedsub}[2]{\ensuremath{\mathop{{}\mathbf{E}}_{#1}\left[#2\right]}\xspace}

\newcommand{\error}{\mathsf{Err}}

\newcommand{\hypothesisc}{\ensuremath{\mathcal H}\xspace}
\newcommand{\hypoC}{\hypothesisc}

\newcounter{definitioncnt}
\newcounter{thmcnt}
\begin{document}
%
\title{\mypapertitle}
\author{\mypaperauthors}
\maketitle


\iffollowingorders
\maketitle
\else
\maketitle
\fi

\begin{abstract}
Many modern machine learning classifiers are shown to be vulnerable to adversarial perturbations of the instances. Despite a massive amount of work focusing on making classifiers robust, the task seems quite challenging. In this work, through a theoretical study, we investigate the adversarial risk and robustness of classifiers and draw a connection to the well-known phenomenon of ``concentration of measure'' in metric measure spaces. We show that if the metric probability space of the test instance is concentrated, any classifier with some initial constant error is inherently vulnerable to adversarial perturbations.

One class of concentrated metric probability spaces are the so-called \Levy families that include many natural distributions. In this special case, our attacks only need to perturb the test instance by at most $O(\sqrt n)$ to make it misclassified, where $n$ is the data dimension. Using our general result about \Levy instance spaces, we first recover as special case some of the previously proved results about the existence of adversarial examples. However, many more \Levy families are known (e.g., product distribution under the Hamming distance) for which we immediately obtain new attacks that find adversarial examples of distance $O(\sqrt n)$.

Finally, we show that concentration of measure for product spaces implies the existence of forms of  ``poisoning'' attacks in which the adversary tampers with the training data with the goal of degrading the classifier. In particular, we show that for any  learning algorithm that uses $m$ training examples, there is an adversary who can increase the probability of any ``bad property'' (e.g., failing on a particular test instance) that initially happens with $1/\poly(m)$ probability to $\approx 1$ by substituting only $\wt{O}(\sqrt m)$ of the examples with other (still correctly labeled) examples.

\end{abstract}

\bigskip
\begin{quote}
{\footnotesize * This is the full version of a work  appearing in AAAI 2019.} \end{quote}

\iffollowingorders
\else
\newpage
\tableofcontents
\fi

\section{ Introduction}
Learning  how to classify instances based on labeled examples is a fundamental task in machine learning. The goal is to find, with high probability, the correct label $c(x)$ of a given test instance $x$ coming from a distribution $\msr$. Thus, we would like to find a good-on-average ``hypothesis'' $h$ (also called the trained model) that  minimizes the error probability $ \Pr_{x \gets \msr}[h(x) \neq c(x)]$, which is referred to as the risk of $h$ with respect to the ground truth $c$.  Due to the explosive use of learning algorithms in real-world systems (e.g., using neural networks for image classification) a more modern approach to the classification problem aims at making the learning process, from training till testing, more \emph{robust}. Namely, even if the instance $x$ is perturbed in a limited way into $x'$ by an adversary $A$,  we would like to have the hypothesis  $h$ still predict the right label for $x'$; hence, minimizing the ``adversarial risk''
$$ \Pr_{x \gets \msr}[h(x') \neq c(x') \text{ for some } x' \text{ ``close'' to } x]$$
of the hypothesis $h$ under such perturbations, where ``close'' is defined by a metric. An attack to increase the risk  is called an ``evasion attack'' (see e.g.,~\cite{biggio2014security,CarliniWagner}) due to the fact that $x'$ ``evades'' the correct classification.
One major motivation behind this problem comes from scenarios such as image classification, in which the adversarially perturbed instance $x'$ would still ``look similar'' to the original $x$, at least in humans' eyes, even though the classifier $h$ might now misclassify $x'$~\cite{GoodfellowEtAl:MakeMLRobust}. In fact, starting with the work of Szegedy et al.~\cite{Szegedy:intriguing} an active line of research (e.g., see \cite{Evasion:TestTime,biggio2014security,Adversarial::Harnessing,Defenses:Distillation,CarliniWagner,Adversarial::FeatureSqueezing})
investigated various attacks and possible defenses to resist such attacks. The race between attacks and defenses in this area motivates a study of whether or not such robust classifiers could ever be found, if they exist at all.

A closely related notion of robustness for a learning algorithm deals with the \emph{training} phase. Here, we would like to know how much the risk of the produced hypothesis $h$ might increase, if an adversary $A$ tampers with the training data $\cT$  
with the goal of increasing the  ``error'' (or any ``bad'' event in general) during the test phase.
Such attacks are  referred to as \emph{poisoning} attacks~\cite{biggio2012poisoning,xiao2015feature,shen2016uror,Mahloujifar2018:ALT,PoisoningOnlineWangChaudhuri}, and the line of research on the power and limitations of poisoning attacks contains numerous attacks and many defenses designed (usually specifically) against them~(e.g., see~\cite{awasthi2014powerjournal,xiao2015feature,shen2016uror,papernot2016towards,rubinstein2009antidote,charikar2017learning,diakonikolas2017statistical,Mahloujifar2018:ALT,diakonikolas2018list,diakonikolas2018sever,prasad2018robust,diakonikolas2018efficient,diakonikolas2017being,diakonikolas2018robustly} and references therein).

The state of affairs in attacks and defenses with regard to the robustness of learning systems in both the evasion and poisoning contexts leads us to our main question:
\begin{quote}
\emph{What are the inherent limitations of defense mechanisms for evasion and poisoning attacks? Equivalently, what are the inherent power of such attacks?}
\end{quote}
Understanding the answer to the above question is fundamental for finding the right bounds that robust learning systems can indeed achieve, and achieving such bounds would be the next natural goal. 

\paragraph{Related prior work.} In the context of evasion attacks, the most relevant to our main question above are the recent works of Gilmer et al.~\cite{gilmer2018adversarial}, Fawzi et al.~\cite{fawzi2018adversarial},  and Diochnos et al.~\cite{Adversarial:NIPS}. In all of these works, isoperimetric inequalities for specific metric probability spaces (i.e., for uniform distributions over the $n$-sphere by \cite{gilmer2018adversarial}, for isotropic $n$-Gaussian by \cite{fawzi2018adversarial}, and for uniform distribution over the Boolean hypercube by \cite{Adversarial:NIPS}) were used to prove that problems on such input spaces are always vulnerable to adversarial instances.\footnote{More formally,  Gilmer et al.~\cite{gilmer2018adversarial} designed specific problems over (two) $n$-spheres, and proved them to be hard to learn robustly, but their proof extend to any problem defined over the uniform distribution over the $n$-sphere. Also, Fawzi et al.~\cite{fawzi2018adversarial} used a different notion of adversarial risk that only considers the hypothesis $h$ and is independent of the ground truth $c$, however their proofs also extend to the same setting as ours.} The work of Schmidt et al.~\cite{schmidt2018adversarially} shows that, at least in some cases, being robust to adversarial instances requires more data. However, the work of Bubeck et al.~\cite{bubeck2018adversarial} proved that \emph{assuming the existence} of classifiers that are robust to evasion attacks, they \emph{could} be found by ``few'' training examples in an information theoretic way. 

In the context of poisoning attacks, some classical results about malicious noise \cite{Valiant::DisjunctionsConjunctions,KearnsLi::Malicious,NastyNoise} could be interpreted as limitations of learning under poisoning attacks. On the positive (algorithmic) side, the   works of Diakonikolas et al.~\cite{diakonikolas2016robust} and Lia et al.~\cite{lai2016agnostic} showed the surprising power of algorithmic robust inference over poisoned data  with error that does not depend on the dimension of the distribution. These works led to an active line of work (e.g., see \cite{charikar2017learning,diakonikolas2017statistical,diakonikolas2018list,diakonikolas2018sever,prasad2018robust,diakonikolas2018efficient} and references therein) exploring the possibility of robust statistics over poisoned data with algorithmic guarantees. The works of \cite{charikar2017learning,diakonikolas2018list} showed how to do \emph{list-docodable} learning, and \cite{diakonikolas2018sever,prasad2018robust} studied supervised learning.

Demonstrating the power of poisoning \emph{attacks}, Mahmoody and Mahloujifar~\cite{pTampTCC17} showed that, assuming an initial $\Omega(1)$ error, a variant of poisoning attacks that tamper with $\approx p$ fraction of the training data \emph{without} using wrong labels (called $p$-tampering) could always increase the error of deterministic classifiers  by $\Omega(p)$ in the targeted poisoning model~\cite{barreno2006can} where the adversary knows the final test instance.  Then Mahloujifar et al.~\cite{Mahloujifar2018:ALT} improved the quantitative bounds of \cite{pTampTCC17} and also applied those attacks to degrade the confidence parameter of any PAC learners under poisoning attacks. Both attacks of \cite{pTampTCC17,Mahloujifar2018:ALT}  were  \emph{online}, in the sense that the adversary does not know the future examples, and as we will see their attack model is very relevant to this work. Koh and Liang \cite{koh2017understanding} studied finding training examples with most influence over the final decision over a  test instance $x$--enabling poisoning attacks. Here, we \emph{prove} the existence of  $O(\sqrt{m})$ examples in the training set that can almost fully degrade the final decision on $x$, assuming $\Omega(1)$ initial error on $x$. 

The work of Bousquet and Elisseeff \cite{bousquet2002stability} studied how specific forms of stability of the hypothesis (which can be seen as robustness under weak forms of ``attacks" that change one training example) imply \emph{standard} generalization (under no attack). Our work, on the other hand, studies \emph{generalization under attack} while the adversary can perturb a lot more (but still sublinear) part of instances.

\paragraph{Other definitions of adversarial examples.} The works of Madry et al.~\cite{madry2017towards} and Schmidt et al.~\cite{schmidt2018adversarially} employ an alternative definition of adversarial risk  inspired by robust optimization. This definition is reminiscent of the definition of ``corrupted inputs'' used by Feige et al.~\cite{feige2015learning} (and related works of \cite{mansour2015robust,feige2018robust,attias2018improved}) as in all of these works, a ``successful'' adversarial example $x'$ shall have a prediction $h(x')$ that is different from the true label of the \emph{original} (uncorrupted) instance $x$. However, such definitions based on corrupted instances  do not always guarantee that the adversarial examples are misclassified. In fact, even going back to the original definitions of adversarial risk and robustness from~\cite{Szegedy:intriguing}, many papers (e.g., the related work of~\cite{fawzi2018adversarial}) only compare the prediction of the hypothesis over the adversarial example with its own prediction on the honest example, and indeed ignore the ground truth defined by the concept $c$.) 
In various ``natural'' settings (such as image classification) the above two definition and ours coincide. We refer the reader to the work of Diochnos et al.~\cite{Adversarial:NIPS} where these definitions are compared and a taxonomy is given, which we will use here as well. 
See Appendix \ref{sec:PC} for more details.

\subsection{Our Results}
In this work, we draw a connection between the general phenomenon of ``concentration of measure'' in metric measured spaces and both evasion and poisoning attacks. A concentrated metric probability space $(\X,\metric,\msr)$ with metric $\metric$ and measure $\msr$ has the property that for any set $\cS$ of measure at least half ($\msr(\cS) \geq 1/2$), most of the points in $\cX$ according to $\msr$, are ``close'' to $\cS$ according to $\metric$ (see Definition \ref{def:conc}). We prove that for any learning problem defined over such a concentrated space, no classifier with an initial constant error (e.g., $1/100$) can be robust to adversarial perturbations. Namely, we prove the following theorem. (See  Theorem \ref{thm:ConctToRisk} for a formalization.)

\begin{theorem}[Informal] \label{thm:mainInf}
Suppose $(\X,\metric,\msr)$ is a concentrated metric probability space from which the test instances are drawn. Then for any classifier $h$ with $\Omega(1)$ initial ``error'' probability, there is an adversary who changes the test instance $x$ into a ``close'' one  and  increases the risk to   $\approx 1$.
\end{theorem}

In Theorem \ref{thm:mainInf}, the  ``error'' could be any undesired event over $h,c,x$ where $h$ is the hypothesis, $c$ is the concept function (i.e., the ground truth) and $x$ is the test instance.

The intuition behind the Theorem \ref{thm:mainInf} is as follows. Let $\cE = \set{x \in \cX \mid h(x) \neq c(x)}$ be the ``error region'' of the hypothesis $h$ with respect to the ground truth concept $c(\cdot)$ on an input space $\cX$. Then, by the concentration property of 
$\cX$ and that $\msr(\cE) \geq \Omega(1)$, we can conclude that at least half of the space $\cX$ is ``close'' to $\cE$, and by one more application of the same concentration property, we can conclude that indeed most of the points in $\cX$ are ``close'' to the error region $\cE$. Thus, an adversary who launches an evasion attack, can indeed push a typical point $x$ into the error region by little perturbations. 
This above argument, is indeed inspired by the intuition behind the previous results of~\cite{gilmer2018adversarial,fawzi2018adversarial}, and~\cite{Adversarial:NIPS} 
all of which use  isoperimetric inequalities for \emph{specific} metric probability spaces to prove limitations of robust classification under adversarial perturbations. Indeed, one  natural way of proving concentration results is to use isoperimetric inequalities that characterize the shape of sets with  minimal boundaries  (and thus minimal measure after expansion). However, we emphasize that bounds on concentration of measure  could be proved even if no such isoperimetric inequalities are known, and e.g., \emph{approximate} versions of such inequalities would also be sufficient. Indeed, in addition to proofs by isoperimetric inequalities, concentration of measure results are proved using tools from various fields such as differential geometry, bounds on eigenvalues of the Laplacian, martingale methods, etc,~\cite{milman1986asymptotic}. Thus, by proving  Theorem \ref{thm:mainInf}, we pave the way for a wide range of results against robust classification for learning problems over \emph{any} concentrated  space. To compare, the results of \cite{gilmer2018adversarial,fawzi2018adversarial,Adversarial:NIPS} have better \emph{constants} due to their use of isoperimetric inequalities, while we achieve similar asymptotic bounds  with worse constants but in broader contexts.


\paragraph{\Levy families.} A well-studied class of concentrated metric probability spaces are the so-called \Levy families (see Definition \ref{def:Levy}) and one special case of such families are known as \emph{normal} \Levy families. In such spaces, when the dimension (seen as the diameter of, or the typical norm of vectors in $(\cX,\metric)$) is $n$, if we expand sets with measure $1/2$ by distance $b$, they will cover   measure at least $1-k_1 \e^{-k_2 b^2 /n}$ for some universal constants $k_1,k_2$.  When translated back into the context of adversarial classification using our Theorem \ref{thm:mainInf}, we conclude that any learning task defined over a normal \Levy metric space $(\cX,\metric,\msr)$ guarantees the existence of (misclassified) adversarial instances that are only $\wt{O}(\sqrt n)$-far from the original instance $x$, assuming that the original error of the classifier is only polynomially large $\geq 1/\poly(n)$.
Interestingly, all the above-mentioned  classifier-independent results  on the existence of adversarial instances  follow as special cases by applying our Theorem \ref{thm:mainInf} to known normal \Levy families (i.e., the $n$-sphere, isotropic $n$-Gaussian, and the Boolean hypercube under Hamming distance). However, many more examples of normal \Levy families are known in the literature (e.g., the unit cube, the unit sphere, the special orthogonal group, symmetric group under Hamming distance, etc.) for which we immediately obtain new results, and in fact, it seems that ``natural'' probabilistic metric spaces are more likely to be  \Levy families than not! In   Section \ref{sec:LevyExamples}, we list some of these examples and give citation where more  examples  could be found.\footnote{More formally, in Definition \ref{def:Levy}, the concentration function is $\e^{-k_2 b^2 \cdot n}$, however in many natural examples that we discuss in Section \ref{sec:LevyExamples}, the original norm required to be a \Levy family is $\approx 1$, while the (expected value of the) ``natural'' norm is $\approx n$ where $n$ is the dimension. (See Remark \ref{rem:rootN}.)}

\paragraph{Robustness against {\em average-case} limited perturbation.} We also prove variants of Theorem \ref{thm:mainInf} that deal with the \emph{average} amount of perturbation done by the adversary with the goal of changing the test instance $x$ into a misclassified $x'$.  Indeed, just like the notion of adversarial risk that, roughly speaking, corresponds to the concentration of metric spaces with a \emph{worst-case} concentration bound, the  robustness of a classifier $h$ with an average-case bound on the perturbations corresponds to the concentration of the metric probability space using an average-case bound on the perturbation.  In this work we introduce the notion of \emph{target-error} robustness in which the adversary targets a specific error probability and plans its (average-case bounded) perturbations accordingly (see Theorem \ref{thm:ConcToRob}).

\paragraph{Relation to hardness of robust image classification.} Since a big motivation for studying the hardness of classifiers against adversarial perturbations comes from the challenges that have emerged in the area of image classifications, here we comment on possible ideas from our work that might be useful for such studies. Indeed, a natural possible approach is to study whether or not the metric measure space of the images is concentrated or not. We leave  such studies for interesting future work. Furthermore, the work of~\cite{fawzi2018adversarial} observed that vulnerability to adversarial instances over  ``nice" distributions (e.g., $n$-Gaussian in their work, and any concentrated distribution in our work) can \emph{potentially} imply attacks on real data \emph{assuming}  that the data is generated with a smooth generative model using the mentioned nice distributions. So, as long as one such mapping could be found for a concentrated space, our impossibility results can potentially be used for deriving similar results about the generated data (in this case image classification) as well.

\paragraph{The special case of product distributions.} One  natural family of metric probability spaces for which Theorem \ref{thm:mainInf} entails new impossibility results are \emph{product} measure spaces under Hamming distance. Results of~\cite{amir1980unconditional,milman1986asymptotic,talagrand1995concentration} show that such metric probability spaces are indeed normal \Levy. Therefore, we immediately conclude that, in any learning task, if the instances come  from any product space of dimension $n$, then an adversary can perturb them to be misclassified  by only changing $O(\sqrt n)$ of the ``blocks'' of the input. A special case of this result covers the case of Boolean hypercube that was recently studied by~\cite{Adversarial:NIPS}. However, here we obtain  impossibilities  for \emph{any} product space. As we will see below, concentration in such spaces are useful beyond evasion attacks.

\paragraph{Poisoning attacks from concentration of product spaces.} One intriguing application of concentration in product measure spaces is to obtain inherent \emph{poisoning} attacks that can attack \emph{any} deterministic learner by tampering with their \emph{training} data and increase their error probability during the (untampered) test phase. Indeed, since the training data is always sampled  as $\cT \gets (\msr,c(\msr))^m$ where $c$ is the concept function and $m$ is the sample complexity, the concentration of the space of the training data under the Hamming distance (in which the alphabet space  is the full space of labeled examples) implies that an adversary can always change the training data $\cT$ into $\cT'$ where $\cT'$ by changing only a ``few'' examples in $\cT$ while producing a classifier $h$ that is more vulnerable to undesired properties. 


\begin{theorem}[Informal] \label{thm:poisInf} Let $L$ be any \emph{deterministic} learning algorithm for a classification task where the confidence of $L$ in producing a ``good'' hypothesis $h$ with error at most $\eps$ is $1-\delta$ for $\delta\geq 1/\poly(m)$. Then, there is always a poisoning attacker who substitutes only $\wt{O}(\sqrt{m})$  of the training data, where $m$ is the total number of examples, with another set of \emph{correctly labeled} training data, and yet degrades the  confidence of the produced hypothesis $h$ to  almost zero. Similarly, an attack with similar parameters can increase the average error of the generated hypothesis $h$ over any chosen test instance $x$ from any initial probability $\geq  1/\poly(m)$ to $\approx 1$.
\end{theorem}

More generally, both attacks of \ref{thm:poisInf} follow as special case of a more general attack in which the adversary can pick any ``bad'' property of the produced hypothesis $h$ that happens with probability at least $\geq 1/\poly(m)$  and increase its chance to hold with probability $\approx 1$ by changing only $\wt{O}(\sqrt{m})$  of the training examples (with other  correctly labeled examples). In fact, by allowing the bad property to be defined over the \emph{distribution} of the produced hypothesis, we will not need $L$ to be deterministic.

Our attacks of Theorem \ref{thm:poisInf} are \emph{offline} in the sense that the adversary needs to know the full training set $\cT$ before substituting some of them. We note that the so-called $p$-tampering attacks of~\cite{Mahloujifar2018:ALT} are \emph{online} in the sense that the adversary can decide about its choices without the knowledge of the upcoming training examples. However, in that work, they could only increase the classification error by $O(p)$ through tampering by $p$ fraction of the training data, while here we get almost full error by only using $p \approx O(\sqrt m)$, which is much more devastating.
\section{Preliminaries}\label{sec:prelims}

\subsection{Basic Concepts and Notation}

\begin{definition}[Notation for metric  spaces]
Let $(\X,\metric)$ be a metric space. We use the notation
$\diam^\metric(\X) = \sup\set{\metric(x,y) \mid x,y \in \X_i}$ to denote the diameter of $\X$ under $\metric$, and we use $\Ball_b^\metric(x)=\set{x' \mid \metric(x,x') \leq b}$ to denote the ball of radius $b$ centered at $x$. When $\metric$ is clear from the context, we simply write $\diam (\X)$ and $\Ball_b (x)$. For a set $\cS\se \X$, by $\metric(x,\cS) = \inf\set{\metric(x,y) \mid  y \in \cS}$  we denote the distance of a point $x$ from $\cS$. 
\end{definition}

Unless stated otherwise, all integrals in this work  are Lebesgue integrals.

\begin{definition}[Nice metric probability spaces] \label{def:niceProb}
We call $(\X, \metric, \msr)$ a metric probability space, if  $\msr$ is a Borel probability measure over $\X$ with respect to the topology defined by $\metric$. Then, for a Borel set $\cE \se \X$, the \emph{$b$-expansion} of $\cE$, denoted by $\cE_b$, is defined as\footnote{The set $\cE_b$ is also called the $b$-flattening or $b$-enlargement of $\cE$, or simply the $b$-ball around $A$.} 
$$\cE_b = \set{x \mid \metric(x,\cE) \leq b}.$$ 
We call $(\X, \metric, \msr)$ a \emph{nice} metric probability space, if the following conditions hold.
\begin{enumerate}
    \item {\bf Expansions are measurable.} For every $\msr$-measurable (Borel) set $\cE \in \cX$, and every $b\geq 0$, its $b$-expansion  $\cE_b$ is also $\msr$-measurable. 
    \item {\bf Average distances exist.} For every two Borel sets $\cE,\cS \in \cX$, the average minimum distance of an element from $\cS$ to $\cE$ exists; namely, the integral $\int_\cS \metric(x,\cE) \cdot d\msr(x)$ exists. 
\end{enumerate}
\end{definition}
At a high level, and as we will see shortly, we need the first condition to define adversarial risk and need the second condition to define (a generalized notion of) robustness. Also, we remark that one can weaken the second condition above based on the first one and still have risk and robustness defined, but since our goal in this work is \emph{not} to do a measure theoretic study, we are willing to make simplifying assumptions that hold on the actual applications, if they make the presentation simpler. 

\subsection{Classification Problems}
\paragraph{Notation on learning problems.}  We  use calligraphic letters (e.g., $\cX$) for sets. 
By $x \gets \msr$ we denote sampling $x$  from the probability measure $\msr$. 
For a randomized algorithm $R(\cdot)$, by $y \gets R(x)$ we denote the randomized execution of $R$ on input $x$ outputting $y$. 
 A classification problem $(\X,\Y,\msr,\C,\hypoC)$ is specified by the following components. 
The set $\X$ is the set  of possible \emph{instances}, 
\Y is the set of possible \emph{labels}, 
$\msr$ is a distribution over $\X$,
$\C$ is a class of \emph{concept} functions where $c\in \C$ is always a mapping from $\X$ to $\Y$.
We did not state the loss function explicitly, as we work with classification problems.
 For  $x \in \X, c \in \C$,
the \emph{risk} or \emph{error} of a hypothesis $h \in \hypoC$ is equal to $\Risk(h,c) = \Pr_{x \gets\msr}[h(x) \neq c(x)]$. We are usually interested in learning problems $(\X,\Y,\msr,\C,\hypoC)$ with a specific metric $\metric$ defined over $\X$ for the purpose of defining risk and robustness under instance perturbations controlled by metric $\metric$. In that case, we simply write $(\X,\Y,\msr,\C,\hypoC,\metric)$ to include $\metric$.  


\begin{definition}[Nice  classification problems]
 We  call $(\X,\Y,\msr,\C,\hypoC,\metric)$ a \emph{nice} classification problem, if the following two conditions hold:
\begin{enumerate}
    \item $(\X, \metric, \msr)$ is a nice metric probability space.
    \item For every $h \in \hypoC,c \in \C$, their error region $\set{x \in \cX \mid h(x) \neq c(x)}$ is $\msr$-measurable.
\end{enumerate}
The second condition above is satisfied, e.g., if  the set of labels $\cY$ (which is usually finite) is countable,  and for all $y \in \cY, f \in \hypoC \cup \C$, the set $\set{x \in \cX \mid f(x)=y} $ is $\msr$-measurable.
\end{definition}

\subsection{The Concentration Function and Some  Bounds}
We now formally define the (standard) notion of concentration function.
\begin{definition}[Concentration function] \label{def:conc}
Let $(\X, \metric, \msr)$ be a metric probability space and $\cE \se \X$ be a Borel set. 
The \emph{concentration function} is then defined as
    $$ \con(b)= 1-\inf \set{\msr(\cE_b) \mid \msr(\cE) \geq 1/2}.$$
\end{definition}

Variations of the following Lemma \ref{lem:talagrand} below are  in \cite{amir1980unconditional,milman1986asymptotic}, but the following version is due to Talagrand \cite{talagrand1995concentration} (in particular, see Equation 2.1.3 of Proposition 2.1.1 in \cite{talagrand1995concentration}).
\begin{lemma}[Concentration of product spaces under Hamming distance]\label{lem:talagrand}
Let $\msr \equiv \msr_1 \times \dots \times \msr_n$ be a product probability measure of dimension $n$ and let the metric be the Hamming distance. For any $\msr$-measurable $\cS \se \X$ such that the $b$-expansion $\cS_b$ of $\cS$ under Hamming distance is also measurable, 
$$\msr(\cS_b) \geq 1- \frac{\e^{-b^2 / n}}{\msr(\cS)} .$$
\end{lemma}

\begin{lemma}[McDiarmid inequality]\label{lem:MCD}
Let  $\msr \equiv \msr_1 \times \dots \times \msr_n$ be a product probability measure of dimension $n$, and let $f\colon \Supp(\msr) \To \R$ be a measurable function such that  $|f(x)-f(y)|\leq 1$ whenever $x$ and $y$  \emph{only} differ in one coordinate. If $a=\Ex_{x \gets \msr}[f(x)]$, then 
$$\Pr_{x \gets \msr}[f(x) \leq a - b] \leq \e^{-2\cdot b^2/n}.$$
\end{lemma}



\section{Evasion Attacks: Finding Adversarial Examples  from Concentration} \label{sec:ER}
In this section, we formally prove our main results about the existence of evasion attacks for learning problems over concentrated spaces. We start by formalizing the notions of risk and robustness.


\begin{definition}[Adversarial risk and robustness] \label{def:ER}
Let $(\X,\Y,\msr,\C,\hypoC,\metric)$ be a nice classification problem. For $h \in \H$ and $c \in \C$, let $\Err = \set{ x \in \X \mid h(x) \neq c(x)}$ be the error region of $h$ with respect to $c$. Then, we define:
\begin{itemize}
    \item {\bf Adversarial risk.}  For   $b \in \Rplus$, the  (error-region)  \emph{adversarial risk} under $b$-perturbation  is 
$$\Risk_b(h,c) = \Pr_{x \gets \msr}\left[\Exists x' \in \Ball_b(x) \cap \Err\right] = \msr(\Err_b).$$
We might call $b$ the ``budget'' of an imaginary ``adversary'' who perturbs $x$ into $x'$. Using $b=0$, we recover the standard notion of risk:  
$\Risk(h,c) = \Risk_0(h,c) = \msr(\Err)$.

\item {\bf Target-error robustness.} Given a  target error $\rho \in (0, 1]$, we define the $\rho$-error robustness as the expected perturbation needed to increase the error to $\rho$; namely,
\begin{align*}
\Rob_\rho(h,c)
&=\inf_{\msr(\cS) \geq \rho}\set{ \expectedsub{x \gets \msr}{\charac_\cS(x) \cdot \metric(x,\Err)}} \\
&= \inf_{\msr(\cS) \geq \rho}\set{\int_\cS \metric(x,\cE) \cdot d \msr(x)},    
\end{align*}
where $\charac_\cS(x)$ is the characteristic function of membership in $\cS$. 
Letting $\rho=1$, we recover the  notion of full robustness 
$\Rob(h,c) = \Rob_1(h,c) = \expectedsub{x \gets \msr}{\metric(x,\Err)}$ that captures the expected amount of perturbations needed to \emph{always} change $x$ into a misclassified $x'$ where $x' \in \cE$.
\end{itemize}
\end{definition}

As discussed in the introduction, starting with~\cite{Szegedy:intriguing}, many papers (e.g., the related work of~\cite{fawzi2018adversarial}) use a definitions of risk and robustness that \emph{only} deal with the hypothesis/model and is independent of the concept function. In \cite{Adversarial:NIPS}, that definition is formalized as ``prediction change'' (PC) adversarial risk and robustness. In Appendix \ref{sec:PC}, we show that using the concentration function $\con(\cdot)$ and our proofs of this section, one can also bound the PC risk and robustness of hypotheses assuming that we have a concentration function. Then, by plugging in any concentration function (e.g., those of \Levy families) and  obtain the desired upper/lower bounds.

In the rest of this section, we focus on misclassification as a necessary condition for the target adversarial example. So, in the rest of this section, we use Definition \ref{def:ER} to prove our results.

\subsection{Increasing Risk and Decreasing Robustness by Adversarial Perturbation}
We now formally state and prove our result that the adversarial risk can be large for any learning problem over concentrated spaces.
Note that, even though the following is stated using the concentration function, having an \emph{upper bound} on the concentration function suffices for using it. Also, we note that all the results of this section extend to settings in which the ``error region'' is substituted with any ``bad'' event modeling an undesired region of instances based on the given hypothesis $h$ and concept function $c$; though the most natural bad event is that error $h(x) \neq c(x)$ occurs.

\begin{theorem}[From concentration  to large adversarial risk] \label{thm:ConctToRisk}
Let $(\X,\Y,\msr,\C,\hypoC,\metric)$ be a nice classification problem. Let $h\in \hypoC$ and $c \in \C$, and let  $\eps = \Pr_{x \gets \msr}[ h(x) \neq c(x)]$ be the error  of  the hypothesis $h$ with respect  to the concept $c$. If 
$\eps > \con(b)$ (i.e., the original error is more than the concentration function for the  budget $b$), then the following two hold.
\begin{enumerate}
    \item {\bf Reaching adversarial risk at least half.}  Using only tampering budget $b$, the adversary can make the adversarial risk to be more than half; namely, $\Risk_b(h,c) > 1/2.$
\item {\bf Reaching adversarial risk close to one.} If in addition we have $\gamma \geq \con(b_2)$, then the adversarial risk for the total tampering budget $b_1+b_2$ is $\Risk_{b_1+b_2}(h,c) \geq 1-\gamma$.
\end{enumerate}
\end{theorem}

\begin{proof}[Proof of Theorem \ref{thm:ConctToRisk}]
Let $\cE = \set{x \in \cX \mid h(x) \neq c(x)} $ be the error region of $(h,c)$, and so it holds that $\eps = \msr(\cE)$.
To prove Part~1, suppose for sake of contradiction that $\Risk_b(h,c) \leq 1/2$. Then, for $\cS=\cX \sm \cE_b$, it holds that $\msr(\cS) = 1-\msr(\cE_b) = 1-\Risk_b(h,c) \geq 1/2$. By the assumption $\msr(\cE) > \con(b)$, we have $\msr(\cS_b) \geq 1-\con(b)>1-\eps$. So, there should be $x \in \cS_b \cap \cE$, which in turn implies that there is a point $y \in \cS$ such that $\metric(y,x) \leq b$. However, that  is a contraction as $\metric(y,x) \leq b$ implies that  $y$ should be in $\cE_b = \cX \sm \cS$.

To prove Part~2, we rely on Part~1. By Part~1, if we use a tampering budget $b_1$, we can increase the adversarial risk to $\Risk_{b_1}(h,c) > 1/2$, but then because of the second assumption $\gamma \geq \con(b_2)$, it means that by using $b_2$ more budget, we can expand the  error region to measure $\geq 1-\gamma$.
\end{proof}

The above theorem provides a general result that applies to \emph{any} concentrated space. So, even though we will compute explicit bounds for spaces such as \Levy families, Theorem \ref{thm:ConctToRisk} could be applied to any other concentrated space as well, leading to stronger or weaker bounds than what \Levy families offer. Now, in the following, we go after finding general relations between the concentration function and the robustness of the learned models.

\paragraph{Simplifying notation.} Suppose $(\X,\metric,\msr)$ is a nice metric probability space. Since our risk and robustness definitions depend only on the error region, for any Borel set $\cE \se\X$ and $b \in \R_+$,  we define its \emph{$b$-tampering   risk} as $\Risk_b(\cE) = \msr(\cE_b),$
and for any such $\cE$ and $\rho \in (0,1]$, we define the $\rho$-error robustness  as $\Rob_{\rho}(\cE) = \inf_{\msr(\cS) \geq \rho} \set{\int_\cS \metric(x,\cE)\cdot d\msr(x)}.$

The following lemma provides a very useful tool for going from adversarial risk to robustness; hence, allowing us to connect concentration of spaces to robustness. In fact, the lemma could be of independent interest, as it states a relation between \emph{worst-case} concentration of metric probability spaces to their \emph{average-case}  concentration with a \emph{targeted} amount of measure to cover.

\begin{lemma}[From adversarial risk to target-error robustness]\label{lem:risktorob-continuous}
For a nice metric probability space $(\X,\metric,\msr)$, let $\cE \se \X$ be a Borel set. If $\rho=\Risk_\ell(\cE)$, then  we have
$$\Rob_\rho(\cE) = \rho \cdot \ell -  \int_{z=0}^{\ell} \Risk_z(\cE)   \cdot dz.$$
\end{lemma}
First, we make a few comments on using Lemma \ref{lem:risktorob-continuous}.
\paragraph{Special case of full robustness.} Lemma \ref{lem:risktorob-continuous} can be used to compute the full robustness also as 
\begin{equation} \label{eq:FullRob}
 \Rob(\cE) = \Rob_1(\cE) =  \ell -  \int_{z=0}^{\ell} \Risk_z(\cE)   \cdot dz,
\end{equation} 
using any $\ell \geq \diam(\X)$, because for such $\ell$ we will have $\Risk_\ell(\cE)=1$. In fact, even if the diameter is not finite, we can always use $\ell=\infty$ and rewrite the two terms as 
\begin{equation} \label{eq:FullRobInf}
    \Rob(\cE) =   \int_{z=0}^{\infty} (1- \Risk_z(\cE))   \cdot dz, 
\end{equation} 
which might or might not converge.

\paragraph{When we only have \emph{lower bounds}  for adversarial risk.} Lemma \ref{lem:risktorob-continuous}, as written, requires the exact amount of risk for the initial set $\cE$. Now, suppose we only have a lower bound  $L_z(\cE) \leq \Risk_z(\cE)$ for the adversarial risk. In this case, we can still use Lemma \ref{lem:risktorob-continuous}, but it will only give us an \emph{upper bound} on the $\rho$-error robustness using any $\ell$ such that $\rho \leq L_\ell(\cE)$ as follows,
\begin{equation} \label{in:ApprRiskToRob}
\Rob_\rho(\cE) 
\leq  \ell -  \int_{z=0}^{\ell} L_z(\cE)   \cdot  dz.
\end{equation}

Note that, even though the above bound looks similar to that of full robustness in Equation \ref{eq:FullRob}, in Inequality \ref{in:ApprRiskToRob} we can use $\ell < \diam(\X)$, which leads to a smaller total bound on the $\rho$-error robustness.


\begin{proof}[Proof of Lemma \ref{lem:risktorob-continuous}]
Let $\nu(\cS) = \int_\cS \metric(x,\cE)\cdot d\msr(x)$. Based on the definition of robustness, we have
$$\Rob_{\rho}(\cE) = \inf_{\msr(\cS) \geq \rho} \left[\nu(\cS)\right].$$
For the fixed $\cE$, let $\maxSE_\cS = \sup\set{\metric(x,\cE):x\in \cS}$,  and
let 
$\comul_\cS:\R\to\R$ be the cumulative distribution function for $\metric(x,\cE)$ over $\cS$, namely $\comul_\cS(z) = \msr(\cE_z \cap \cS)$.  Whenever $\cS$ is clear from the context we simply write $\maxSE=\maxSE_\cS, \comul(\cdot) = \comul_\cS(\cdot)$. 
Before continuing the proof, we prove the following claim.

\begin{claim}\label{clm:IntegralLemma}
Let $F$ be a cumulative distribution function of a random variable. For any $m\in \R^+$,
    $$\int_{z=0}^{m} z \cdot d \comul(z) + \int_{z=0} ^{m} \comul(z) \cdot dz = m \cdot \comul(m)$$
where the left integral shall be interpreted as  Lebesgue integral over the Lebesgue–Stieltjes measure associated with the  cumulative distribution function $F(\cdot)$.
\end{claim}
\begin{proof}[Proof of Claim \ref{clm:IntegralLemma}]
 Claim \ref{clm:IntegralLemma} follows from the integration-by-parts (extension) for Lebesgue integral over the Lebesgue–Stieltjes measure.\remove{ In fact, for general $G(z)$, holds that LS integral $\int_{z=0}^{m}  G(z) \cdot d \comul(z)$  exists iff $\int_{z=0} ^{m} \comul(z) \cdot dz $ exists, and in that case their sum will be $ \comul(m)\cdot G(m) - \comul(0) \cdot G(0)$. In our case, $G(z)=z$, and $\int_{z=0} ^{m} \comul(z) \cdot dz$ is even Riemann integrable as $\comul(z)$ is monotone. Therefore, both integrals exist and their summation is  $\comul(m)\cdot G(m) - \comul(0) \cdot G(0) = m \cdot \comul(m) - 0$.}
\end{proof}
Now, we have
\begin{align*}
\nu(\cS)&=\int_\cS \metric(x,\cE)\cdot d\msr(x) =\int_{z=0}^{m} z\cdot d\comul(z) \\
\text{(by Claim \ref{clm:IntegralLemma})}~~~&=m\cdot \comul(m) - \int_{z=0}^{m} \comul(z)\cdot dz . 
\end{align*}
Indeed, for the special case of $\cS= \cE_\ell$ 
we have $\maxSE_\cS = \ell,  \comul_\cS(\maxSE_\cS) = \comul_\cS(\ell) =  \msr(\cS) =\rho$. Thus,
$$\nu(\cE_\ell) = \maxSE_{\cE_\ell}\cdot \comul_{\cE_\ell}(\maxSE_{\cE_\ell}) - \int_{z=0}^{\maxSE_{\cE_\ell}} \comul_{\cE_\ell}(z)\cdot dz = \ell \cdot \rho -  \int_{z=0}^{\ell} \Risk_z(\cE)\cdot dz ,$$ 
and so the robustness can be bounded from above as
\begin{align}\label{eq:001}
\Rob_\rho(\cE) =  \inf_{\msr(\cS) \geq \rho} \left[\nu(\cS)\right] \leq \nu(\cE_\ell)
= \ell \cdot \rho -  \int_{z=0}^{\ell} \Risk_z(\cE)\cdot dz.
\end{align}

We note that, if we wanted to prove Lemma \ref{lem:risktorob-continuous} for the \emph{special} case of \emph{full} robustness (i.e., $\ell \geq \diam(\cX), \msr(\cE_\ell)= \rho=1$), the above concludes the proof. The rest of the proof, however, is necessary for the more interesting case of target-error robustness. At this point, all we have to prove is a similar \emph{lower} bound for any $\cS$ where  $\msr(\cS) \geq \rho$, so in the following assume $\cS$ is one such set. 
By definition, it holds that
\begin{equation}\label{ineq:subset}
    \Forall z\in [0,m], \comul(z) \leq \msr(\cE_z) = \Risk_{z}(\cE)
\end{equation}
and
\begin{equation}\label{ineq:bydef}
\comul(m) = \msr(\cS) \geq \rho.
\end{equation}
First, we show that
\begin{equation}\label{ineq:positive}
    \int_{z=\ell}^{m} (\comul(z)-\comul(m ))\cdot dz \leq 0 .
\end{equation}
 The inequality above clearly holds if $\maxSE\geq \ell$. 
 We prove that if $\ell>\maxSE$ then the integral is equal to 0. We know that $\comul(\ell)\leq \msr(\cE_\ell)= \rho$, therefore $\comul(m)\geq \rho \geq \comul(\ell)$. We also know that $\comul$ is an increasing function and $\ell>\maxSE$ therefore $\comul(m)=\rho= \comul(\ell)$.  So we have $\forall z\in[m,\ell], \comul(z)=\rho$ which implies 
 $\int_{z=\ell}^{m} (\comul(z)-\comul(m ) )\cdot dz = 0$.
Now, we get 
\begin{align*}
 \nu(\cS) &=m\cdot \comul(m) - \int_{z=0}^{m} \comul(z)\cdot dz\\
  & = \ell\cdot \comul(m)  - \int_{z=0}^{\ell} \comul(z)\cdot dz  - \int_{z=\ell}^{m} (\comul(z)-\comul(m ))\cdot dz\\
 \text{(by Inequality \ref{ineq:bydef})~~}&\geq \ell\cdot\rho - \int_{z=0}^{\ell} \comul(z)\cdot dz  - \int_{z=\ell}^{m} (\comul(z)-\comul(m ))\cdot dz\\
 \text{(by Inequality \ref{ineq:subset})~~} &\geq \ell\cdot\rho  - \int_{z=0}^{\ell} \Risk_z(\cE)\cdot dz -\int_{z=\ell}^{m} (\comul(z)-\comul(m ))\cdot dz\\
  \text{(by Inequality \ref{ineq:positive})~~}&\geq \ell\cdot\rho - \int_{z=0}^{\ell} \Risk_z(\cE)\cdot dz .
\end{align*}
The above lower bound on $\Rob_\rho(\cE)$ and the upper bound of  Inequality \ref{eq:001} conclude the proof.
\end{proof}

We now formally state our result  that concentration in the instance space leads to small robustness of classifiers.
Similarly to Theorem \ref{thm:ConctToRisk}, we note that even though the following theorem is stated using the concentration function, having an upper bound on the concentration function would suffice.

\begin{theorem}[From concentration to small robustness] \label{thm:ConcToRob} Let $(\X,\Y,\msr,\C,\hypoC,\metric)$ be a nice classification problem. Let $h\in \hypoC$ and $c \in \C$, and let  $\eps = \Pr_{x \gets \msr}[ h(x) \neq c(x)]$ be the error  of  the hypothesis $h$ with respect  to the concept $c$. Then if $\eps >  \con(b_1)$  and $1-\rho \geq \con(b_2)$, we have
\begin{align*}
\Rob_\rho(\cE) 
& \leq  (1 -\eps)\cdot b_1  + \int_{z=0}^{b_2} \con(z)\cdot dz.
\end{align*}
\end{theorem}

\begin{proof}[Proof of Theorem \ref{thm:ConcToRob}]
By Theorem \ref{thm:ConctToRisk}, we know that $\Risk_{b_1}(\cE)=\msr(\cE_{b_1})\geq \frac{1}{2}$ which implies $\Risk_{b_1+b_2}(\cE) = \Risk_{b_2}(\cE_{b_1}) \geq \rho$. If we let $\rho^*=\Risk_{b_1+b_2}(\cE)$, then we have
\begin{align*}
\Rob_\rho(\cE) &\leq \Rob_{\rho^*}(\cE)\\
\text{(by Lemma \ref{lem:risktorob-continuous})~~}&= \int_{z=0}^{b_1+b_2} \left(\rho^* -\Risk_z(\cE)\right)\cdot dz\\
&= \int_{z=0}^{b_1} \left(\rho^* -\Risk_z(\cE)\right)\cdot dz+ \int_{b_1}^{b_1+b_2} \left(\rho^* -\Risk_z(\cE)\right)\cdot dz\\
&\leq (\rho^* -\gamma)\cdot b_1 + \int_{b_1}^{b_1+b_2}\left(\rho^* -\Risk_z(\cE)\right)\cdot dz\\
&= (\rho^* -\gamma)\cdot b_1 + \int_{z=0}^{b_2}\left(\rho^* -\Risk_z(\cE_{b_1})\right)\cdot dz\\
\text{(by Theorem \ref{thm:ConctToRisk})~~}&\leq (\rho^* -\gamma)\cdot b_1 + \int_{z=0}^{b_2}\left(\rho^* -1 + \con(b_2)\right)\cdot dz\\
&=(\rho^* -\eps)\cdot b_1 + (\rho^* -1)\cdot b_2 + \int_{z=0}^{b_2} \con(z)\cdot dz \\
&\leq (1 -\eps)\cdot b_1  + \int_{z=0}^{b_2} \con(z)\cdot dz.
\end{align*}
\end{proof}

\subsection{Normal \Levy Families as Concentrated Spaces} \label{sec:Levy}

In this subsection, we study a well-known  special case of concentrated spaces called normal \Levy families, as a rich class of concentrated spaces, leading to specific bounds on the risk and robustness of learning problems whose test instances come from any normal \Levy family. We start by formally defining normal \Levy families. 

\begin{definition}[Normal \Levy families] \label{def:Levy}
A  family of  metric probability spaces $(\X_n, \metric_n, \msr_n)_{i\in\N}$ with corresponding concentration functions $\con_n(\cdot)$ is  called a \emph{$(k_1,k_2)$-normal \Levy family}  if 
$$\con_n(b) \leq k_1 \cdot \e^{-k_2\cdot b^2 \cdot n}.$$
\end{definition}

The following theorem shows that classifying instances that come from a normal \Levy family
has the inherent vulnerability to perturbations of size $O(1/\sqrt n)$

\begin{theorem}[Risk and robustness in normal \Levy families] \label{thm:RiskRobLevy}
Let $(\X_n,\Y_n,\msr_n,\C_n,\hypoC_n,\metric_n)_{n \in \N}$ be a nice classification problem  with a metric probability space $(\X_n,\metric_n,\msr_n)_{n \in \N}$ that is a  $(k_1,k_2)$-normal \Levy family. Let $h\in \hypoC_n$ and $c \in \C_n$, and let  $\eps = \Pr_{x \gets \msr}[ h(x) \neq c(x)]$ be the error  of  the hypothesis $h$ with respect  to the concept $c$. 
\begin{enumerate}
    \item {\bf Reaching adversarial risk at least half.}  If $b > {\sqrt{\ln({k_1}/{\eps})}} /{\sqrt{k_2 \cdot n}} $, then $\Risk_{b}(h,c) \geq 1/2$.
    \item {\bf Reaching Adversarial risk close to one.} If $b > {\sqrt{\ln({k_1}/{\eps}) + \ln({k_1}/{\gamma})}} /{\sqrt{k_2 \cdot n}} $, then it holds that $\Risk_{b}(h,c) \geq 1-\gamma$.
    
    \item {\bf Bounding target-error robustness.} For any $\rho\in[\frac{1}{2},1]$, we have
    $$\Rob_{\rho}(h,c) \leq \frac{(1 -\eps) \sqrt{{\ln({k_1}/{\eps})}} +  \erf\big(\sqrt{\ln({k_1}/{(1-\rho)})}\big)\cdot {k_1 \sqrt{\pi}}/{2}} {\sqrt{k_2 \cdot n}}.$$
\end{enumerate}
\end{theorem}
\begin{proof}[Proof of Theorem \ref{thm:RiskRobLevy}]
Proof of Part 1 is similar to (part of the proof of) Part 2, so we focus on Part 2. 

To prove Part 2, let $b_2 = \sqrt{\frac{\ln(k_1/\gamma)}{k_2 \cdot n}}$ and $b_1 = b- b_2 > \sqrt{\frac{\ln(k_1/\eps)}{k_2 \cdot n}}$. Then, we get $k_1  \cdot \e^{-k_2  \cdot b_2^2 \cdot n} = \gamma$ and $k_1  \cdot \e^{-k_2 \cdot b_1^2 \cdot n} >\eps$. Therefore, by directly using Part 2 of Theorem~\ref{thm:ConctToRisk} and Definition~\ref{def:Levy} (of normal \Levy families), we conclude that $\Risk_{b}(h,c) \geq 1-\gamma$ for $b=b_1+b_2$.\\

We now prove Part 3. By Theorem~\ref{thm:ConcToRob}, we have
\begin{align*}
\Rob_{\rho}(h,c) 
\leq (1 -\eps)\cdot b_1+ k_1\cdot\int_0^{b_2} \e^{-k_2\cdot z^2 \cdot n}\cdot dz=  (1 -\eps)\cdot b_1 + \frac{k_1\cdot\sqrt{\pi}}{2 \sqrt{n\cdot k_2}}\cdot \erf\left(b_2\cdot\sqrt{ n \cdot k_2}\right).
\end{align*}
\end{proof}

Here we remark on its interpretation in an asymptotic sense, and discuss how much initial error is needed to achieve almost full adversarial risk.

\begin{corollary}[Asymptotic risk and robustness in normal \Levy families] \label{cor:RiskRobLevy}
Let $\mathsf{P}_n$ be a nice classification problem defined over a metric probability space that is a normal \Levy family, and let $\eps$ be the error probability of a hypothesis $h$ with respect to some concept function $c$.
\begin{enumerate}
    \item  {\bf Starting from constant error.} If $\eps \geq \Omega(1)$, then for any constant $\gamma$, one can get adversarial risk $1-\gamma$ for $h$ using only $O(1/\sqrt n)$ perturbations, and  full robustness of $h$ is also $O(1/\sqrt n)$.
    \item {\bf Starting from sub-exponential error.} If $\eps \geq \exp(-o(n))$,  then  one can get adversarial risk $1-\exp(-o(n))$ for $h$ using only $o(1)$ perturbations, and  full robustness is also $o(1)$.
\end{enumerate}
\end{corollary}


\begin{remark}[How much perturbation is needed? $O(\sqrt n)$ or $O(1/\sqrt n)$?] \label{rem:rootN}
The amount of perturbation in normal \Levy families needed to (almost certainly) misclassify the adversarial example is $O(1/\sqrt n)$, but this is also the case that ``typically'' metric probability spaces become normal \Levy under a ``normalized'' metric; meaning that the diameter (or more generally the average of distances of random pairs)  is $\Theta(1)$. (E.g., when working with the unit $n$-sphere.) However, in some occasions, the ``natural'' metrics over those spaces is achieved by scaling up the typical distances to $\Theta(n)$ (e.g., the Hamming distance in the Boolean hypercube). In that case, the bounds of Theorem \ref{thm:RiskRobLevy} also get scaled up to $O(\sqrt n)$ (for constants $\eps,\gamma$).
\end{remark}


\subsubsection{Examples of Normal \Levy Families.} \label{sec:LevyExamples}
Here, we list some natural metric probability spaces that are known to be normal \Levy families. For the references and more examples we refer the reader to excellent sources \cite{ledoux2001concentration,giannopoulos2001euclidean,milman1986asymptotic}. There are other variants of \Levy families, e.g., those called \Levy (without the adjective ``normal'') or  \emph{concentrated} \Levy families \cite{alon1985lambda1} with stronger concentration, but we skip them and refer the reader to the cited sources and general tools of Theorems~\ref{thm:ConctToRisk} and~\ref{thm:ConcToRob} on how to apply \emph{any} concentration of measure results  to get bounds on risk and robustness of classifiers.
\begin{itemize}
    \item {\bf Unit sphere under Euclidean or Geodesic distance}.  The unit $n$-spheres $\mathbb{S}^n$ (of radius $1$ in $\R^{n+1}$), under the geodesic distance (or Euclidean distance) and the  normalized  rotation-independent uniform measure is a normal \Levy family.  \Levy was first \cite{levy1951problemes} to notice that the isoperimetric inequality for $\mathbb{S}^n$ makes it (what is now known as a) \Levy family.

    \item {\bf $\R^n$ under Gaussian distribution and Euclidean distance.} $\R^n$ with Euclidean distance and $n$-dimensional   Gaussian measure (where expected Euclidean length is $1$) is a normal \Levy family. This follows from the Gaussian isoperimetric inequality \cite{borell1975brunn,sudakov1978extremal}.

    \item {\bf Unit cube and unit ball under Euclidean distance.} Both the unit cube $[0,1]^n$ and the unit $n$-ball (of radius $1$) are normal \Levy families under normalized Euclidean distance (where the diameter is $1$) and normalized Lebesgue distributions (see Propositions 2.8 and 2.9 in \cite{ledoux2001concentration}).

    \item {\bf Special orthogonal group.} The special orthogonal group $\mathrm{SO}(n)$ (i.e., the subgroup of the orthogonal group $\mathrm{O}(n)$ containing matrices with determinant one) equipped with the Hilbert-Schmidt metric and the Haar probability measure is a normal \Levy family.

    \item {\bf Product distributions under Hamming distance.} Any product distribution $\msr^n$ with normalized Hamming distance is a normal \Levy family \cite{amir1980unconditional,milman1986asymptotic,talagrand1995concentration}.
    In particular, the Boolean hypercube $\bits^n$ with normalized Hamming distance and uniform distribution is a normal \Levy family \cite{amir1980unconditional}.\footnote{This also follows from the isoperimetric inequality of \cite{harper1966optimal}.} In the next section, we will use the concentration of product spaces to obtain \emph{poisoning} attacks against learners.
    \item {\bf Symmetric group under Hamming distance.} The set of all permutations $\Pi^n$ under Hamming distance and the uniform distribution forms a \emph{non-product}  \Levy family. 
\end{itemize}

\section{Poisoning Attacks from Concentration of Product Measures}

In this section, we design new poisoning attacks against any deterministic learning algorithm, by using the concentration of space in the domain of training data. We start by defining the confidence and error parameters of learners.

\subsection{Definition of Confidence and Chosen-Instance Error}

\begin{definition}[Probably approximately correct learning]
An algorithm $L$ is an $(\eps(\cdot),\delta(\cdot))$-PAC learner for a classification problem $(\X, \Y, \msr, \h ,\C)$, if for all $c\in\C$ and $m\in\N$, we have
$$\Pr_{\substack{\cT \gets \left(\msr,c(\msr)\right)^m\\ h \gets L(\cT)}}[\Risk(h,c) > \eps(m)] \leq \delta(m).$$
The function $\eps(\cdot)$ is the error parameter, and $1-\delta(m)$ is the confidence of the learner $L$.
\end{definition}

Now, we formally define the class of poisoning attacks and their properties.
\begin{definition}[Poisoning attacks]
 Let $(\X, \Y, \msr, \h ,\C)$ be a classification with a learning algorithm $L$. Then, a poisoning adversary $A$ for $(L,\X, \Y, \msr, \h ,\C)$ is an algorithm  that takes as input  a training set $\cT\gets (\msr,c(\msr))^m$ and outputs a modified training set $\cT' = A(\cT)$ of the same size\footnote{Requiring the sets to be equal only makes our \emph{negative} attacks \emph{stronger}.}. We also interpret $\cT$ and $\cT$ as vectors with $m$ coordinates with a large alphabet and  let $\HD$ be  the Hamming distance for such vectors of $m$ coordinates. For any $c\in \C$, we define the following properties for $A$. 
\begin{itemize}
    \item $A$ is called \emph{plausible} (with respect to $c$), if $y=c(x)$ for all $(x,y) \in \cT'$.
    \item $A$ has \emph{tampering budget} $b \in [m]$  if for all $ \cT \gets (\msr,c(\msr))^m,\cT' \gets A(\cT)$, we have
    $$\HD(\cT', \cT) \leq b.$$
    \item $A$ has \emph{average} tampering budget $b$, if we have:
    $$\Ex_{\substack{\cT\gets (\msr,c(\msr))^m \\ \cT' \gets A(\cT)}}[\HD(\cT', \cT))] \leq b.$$
\end{itemize}
\end{definition}

Before proving our results about the power of poisoning attacks, we need to define the confidence function of a learning algorithm under such attacks.

\begin{definition}[Confidence  function and its adversarial variant]
For a learning algorithm $L$ for a classification problem $(\X, \Y, \msr, \h ,\C)$, we use $\conf_A$ to define the \emph{adversarial confidence} in the presence of a poisoning adversary $A$. Namely,
$$\conf_A(m,c,\eps)=\Pr_{\substack{\cT \gets \left(\msr,c(\msr)\right)^m\\ h \gets L(A(\cT))}}[\Risk(h,c) \leq \eps].$$

By $\conf(\cdot)$, we denote $L$'s \emph{confidence function} without any attack; namely, $\conf(\cdot) = \conf_I(\cdot)$ for the trivial (identity) attacker $I$ that does not change the training data.
\end{definition}

\begin{definition}[Chosen instance (average) error and its adversarial variant]\label{def:SPpac} For a classification problem $(\X,\Y,\msr, \hypoC, \C)$, and a learning algorithm $L$, a chosen instance $x\in\X$, a concept $c\in\C$ and for some $m \in \N$, the \emph{chosen-instance} error of $x$ in presence of a poisoning adversary $A$ is 
$$\error_A(m,c,x)=\Pr_{\substack{\cT \gets (\msr,c(\msr))^m\\ h \gets L(A(\cT))}}[h(x)\neq c(x)].$$
The \emph{chosen-instance} error for $x$ (without attacks) is then defined as $\error(m,c,x)=\error_I(m,c,x)$ using the trivial adversary that outputs its input.
\end{definition}

\subsection{Decreasing Confidence and Increasing Chosen-Instance Error through Poisoning}
The following theorem formalizes (the first part of) Theorem \ref{thm:poisInf}. We emphasize that by choosing the adversary \emph{after} the concept function is fixed, we allow the adversary to depend on the concept class. This is also the case in e.g., $p$-tampering poisoning attacks of \cite{Mahloujifar2018:ALT}. However, there is a big distinction between our attacks here and those of \cite{Mahloujifar2018:ALT}, as our attackers need to know the \emph{entire} training sequence before tampering with them, while the attacks of \cite{Mahloujifar2018:ALT} were online.

\begin{theorem}\label{Thm:poisoning}
For any classification problem $(\X, \Y, \msr, \h ,\C)$, let $L$ be a deterministic learner, $c\in \C$ and $\eps \in [0,1]$. Also let $\conf(m,\eps,c)= 1-\delta$ be the original confidence  of $L$ for error probability $\eps$.
\begin{enumerate}
    \item 
    For any $\gamma \in [0,1]$, there is a plausible poisoning adversary $A$ with tampering budget at most $\sqrt{-\ln(\delta \cdot \gamma) \cdot m}$ such that, $A$ makes the adversarial confidence to be as small as $\gamma$:
$$\conf_A(\eps,c,m) \leq \gamma.$$ 
 \item  There is a plausible poisoning adversary $A$ with \emph{average} tampering budget $\sqrt{-\ln( \delta)  \cdot m/2}$ eliminating all the confidence:
$$\conf_A(\eps,c,m)= 0.$$  
\end{enumerate}
\end{theorem}

Before proving  Theorem \ref{Thm:poisoning}, we introduce  a notation.
\paragraph{Notation.} For $\xVec = (x_1,\dots,x_m) \in \X^m$ we use $(\xVec,c(\xVec))$ to denote $\big((x_1,c(x_1)),\dots,(x_m,c(x_m))\big)$.

\begin{proof}[Proof of Theorem \ref{Thm:poisoning}]
We first prove Part~1.
Let $\cF=\set{\xVec\in \X^m\mid L((\xVec,c(\xVec)) = h, \Risk(h, c) > \eps}$, and let $\cF_b$ be the $b$ expansion of $\cF$ under Hamming distance inside $\X^m$.

We now define an adversary $A$ that fulfills the statement of Part~1 of Theorem \ref{Thm:poisoning}. Given a training set $\cT=(\xVec, c(\xVec))$, the adversary $A$ does the following.
\begin{itemize}
    \item [Case 1:] If $\xVec\in \cF_b$, it selects an arbitrary $\xVec' \in \cF$ where $\HD(\xVec,\xVec') \leq b$ and outputs $\cT'=(\xVec',c(\xVec'))$.
    \item [Case 2:] If $\cT\not\in \cF_b$, it does nothing and outputs $\cT$.
\end{itemize}
By definition, $A$ is using tampering budget at most $b$, as its output is always in a Hamming ball  of radius $b$ centered at $\xVec$. In addition, $A$ is a plausible attacker, as it always uses correct labels. 

We now show that $A$ decreases the confidence as stated. Note that by the definition of $\conf$ we have $ \conf(\eps,c,m) = \msr^{(m)}(\cF)$ where $ \msr^{(m)}$ is the product distribution measuring according to  $m$  independently samples from $\msr$. By Lemma \ref{lem:talagrand}, we have $\msr^{(m)}(\cF_b) \geq 1-\frac{\e^{-b^2/m}}{\msr^{(m)}(\cF)}$ which by letting $b=\sqrt{-\ln(\delta \cdot \gamma) \cdot m}$ implies that
\begin{equation}\label{ineq:300}
\msr^{(m)}(\cF_b) \geq 1-\gamma.
\end{equation}
We also know that if $A$ goes to Case 1, it always selects some $\xVec'\in \cF$, and that  means that the generated hypothesis using $A$'s output will have a $\Risk$ greater than or equal to $\eps$. Also, if $A$ goes to Case 2 then it will output the original training set which means the generated hypothesis will have a $\Risk$ less than $\eps$. Therefore, we have
$$\conf_A(\eps,c,m) = \Pr_{\xVec\gets \msr^{(m)}}[\text{Case 1}] = \msr^{(m)}(\cF_b) \geq 1-\gamma.$$

Before proving Part~2, we state the following claim, which we prove using McDiarmid Inequality.
\begin{claim}\label{clm:prod-rob}
For any product distribution $\mssr=\mssr_1\times\dots\times \mssr_m$ where $(\Supp(\mssr), \HD, \mssr)$ is a nice metric probability space and any set $\cS \subseteq \Supp(\mssr)$ where $\mssr(\cS)=\eps$, we have
$$\Ex_{\xVec\gets \mssr}[\HD(\xVec,\cS)] \leq \sqrt{\frac{-\ln(\eps)\cdot m}{2}}.$$
\end{claim}
\begin{proof}[Proof of Claim \ref{clm:prod-rob}]
We define function $f(\xVec) =\HD(\xVec,\cS)$. Because $(\Supp(\mssr), \HD, \mssr)$ is a nice metric probability space by assumption, $f$ is a measurable function.  Moreover, it is easy to see that for every pair $(\xVec,\xVec')$ we have $|f(\xVec)-f(\xVec')|\leq \HD(\xVec,\xVec')$ (i.e., $f$ is Lipschitz). Now if we let $a=\Ex_{\xVec\gets \mssr}[f(\xVec)]$, by using Lemma $\ref{lem:MCD}$, we get
$$\eps = \mssr(\cS) = \Pr_{\xVec\gets \mssr}[f(\xVec) = 0] = \Pr_{\xVec\gets \mssr}[f(\xVec) \leq 0] \leq \e^{-2a^2/m}$$
simply because for all $\xVec\in \cS$ we have $f(\xVec)=0$. Thus, we get $a\leq \sqrt{{-\ln(\eps)\cdot m}/{2}}$. 
\end{proof}

Now we prove Part~2. We define an adversary $A$ that fulfills the statement of the second part of the theorem. Given a training set $\cT=(\xVec,c(\xVec))$ the adversary selects some $\xVec'\in \cF$ such that $\HD(\xVec,\xVec') = \HD(\xVec,\cF)$ (i.e., one of the closest points in $\cF$ under Hamming distance). The adversary then outputs $\cT'=(\xVec',c(\xVec'))$. It is again clear that this attack is plausible, as the tampered instances are still within the support set of the correct distribution. Also, it is the case that $\conf_A(\eps,c,m)=0$, as the adversary always selects $\xVec'\in\cF$. To bound the average budget of $A$ we use Claim \ref{clm:prod-rob}. By the description of $A$, we know that the average number of changes that $A$ makes to $\xVec$ is equal to $\Ex_{\xVec\gets \msr^{(m)}} [\HD(\xVec, \cF)]$ which, by Claim \ref{clm:prod-rob}, is bounded by $\sqrt{{-\ln(\eps)\cdot m}/{2}}$.
\end{proof}

\begin{remark}[Attacks for any undesired predicate]
As should be clear from the proof of Theorem~\ref{Thm:poisoning}, this proof directly extends to any setting in which the adversary wants to increase the probability of any ``bad'' event $B$ defined over the hypothesis $h$, if $h$ is  produced deterministically based on the training set $\cT$. More generally, if the learning rule is not deterministic, we can still increase the probability of any bad event $B$ if $B$ is defined directly over the training data $\cT$. This way, we can increase the probability of bad predicate $B$, where $B$ is defined over the \emph{distribution} of the  hypotheses.
\end{remark} 

 We now state our results about the power of poisoning attacks that increase the \emph{average} of the error probability of learners. Our attacks, in this case, need to know the final text instance $x$, which makes our attacks \emph{targeted} poisoning attacks~\cite{barreno2006can}.



\begin{theorem} \label{thm:targeted-poisoning}
For any classification problem $(\X, \Y, \msr, \h ,\C)$, let $L$ be a deterministic learner, $x\in\X$, $c\in \C$, and let $\eps = \error(m,c,x)$ be the chosen-instance error of $x$ without any attack. 
\begin{enumerate}
    \item For any  $\gamma\in(0,1]$, there is a plausible poisoning adversary $A$ with budget $\sqrt{-\ln(\eps\cdot\gamma)\cdot m}$ such that
$$\error_A(m,c,x)\geq 1-\gamma.$$ 
 \item There is a plausible poisoning adversary $A$ with average budget$\sqrt{-\ln(\eps)\cdot m}$  such that 
$$\error_A(m,c,x)= 1.$$  
\end{enumerate}
\end{theorem}

\begin{proof}[Proof of Theorem \ref{thm:targeted-poisoning}]
The proof is very similar to the proof of Theorem \ref{Thm:poisoning}. We only have to change the description of $\cF$ as  $$\cF = \set{\xVec\in \X^m \mid h = L(\xVec,c(\xVec)), h(x)\neq c(x)},$$
 and then everything directly extends to the new setting.
\end{proof}


First now  remark on the power of poisoning attacks of Theorems \ref{Thm:poisoning} and \ref{thm:targeted-poisoning}. 

\begin{remark}[Asymptotic power of our poisoning  attacks]
We note that, in  Theorem \ref{Thm:poisoning}  as long as the initial confidence is $1-1/\poly(n)$, an adversary can decrease it to $1/\poly(n)$ (or to $0$, in the average-budget case) using only tampering budget $\wt{O}(\sqrt n)$. Furthermore, if the initial confidence is at most $1-\exp(-o(n))$ (i.e., subexponentially far from $1$) it can be made subexponentially small $\exp(-o(n))$ (or even $0$, in the average-budget case) using only a sublinear $o(n)$ tampering budget. The same remark holds for  Theorem \ref{thm:targeted-poisoning} and average error. Namely, if the initial average error for a test example is $1/\poly(n)$, an adversary can decrease increase it to $1-1/\poly(n)$ (or to $1$, in the average-budget case) using only tampering budget $\wt{O}(\sqrt n)$, and if the initial average error is at least $\exp(-o(n))$ (i.e., subexponentially large), it can be made subexponentially close to one: $1- \exp(-o(n))$ (or even $1$, in the average-budget case) using only a sublinear $o(n)$ tampering budget. The damage to average error is even more devastating, as typical PAC learning arguments usually do not give anything more than a $1/\poly(n)$ error.
\end{remark}



\iffollowingorders
\bibliographystyle{aaai}
\else
\bibliographystyle{alpha}
\fi

\bibliography{Biblio/references,Biblio/abbrev0,Biblio/crypto,Biblio/CryptoCitations2,Biblio/CryptoCitations}

\newcommand{\etalchar}[1]{$^{#1}$}
\begin{thebibliography}{DKK{\etalchar{+}}18b}

\bibitem[ABL17]{awasthi2014powerjournal}
Pranjal Awasthi, Maria{-}Florina Balcan, and Philip~M. Long.
\newblock {The Power of Localization for Efficiently Learning Linear Separators
  with Noise}.
\newblock {\em Journal of the {ACM}}, 63(6):50:1--50:27, 2017.

\bibitem[AKM18]{attias2018improved}
Idan Attias, Aryeh Kontorovich, and Yishay Mansour.
\newblock Improved generalization bounds for robust learning.
\newblock {\em arXiv preprint arXiv:1810.02180}, 2018.

\bibitem[AM80]{amir1980unconditional}
D~Amir and VD~Milman.
\newblock Unconditional and symmetric sets inn-dimensional normed spaces.
\newblock {\em Israel Journal of Mathematics}, 37(1-2):3--20, 1980.

\bibitem[AM85]{alon1985lambda1}
Noga Alon and Vitali~D Milman.
\newblock $\lambda$1, isoperimetric inequalities for graphs, and
  superconcentrators.
\newblock {\em Journal of Combinatorial Theory, Series B}, 38(1):73--88, 1985.

\bibitem[BCM{\etalchar{+}}13]{Evasion:TestTime}
Battista Biggio, Igino Corona, Davide Maiorca, Blaine Nelson, Nedim Srndic,
  Pavel Laskov, Giorgio Giacinto, and Fabio Roli.
\newblock {Evasion Attacks against Machine Learning at Test Time}.
\newblock In {\em ECML/PKDD}, pages 387--402, 2013.

\bibitem[BE02]{bousquet2002stability}
Olivier Bousquet and Andr{\'e} Elisseeff.
\newblock Stability and generalization.
\newblock {\em Journal of machine learning research}, 2(Mar):499--526, 2002.

\bibitem[BEK02]{NastyNoise}
Nader~H. Bshouty, Nadav Eiron, and Eyal Kushilevitz.
\newblock {{PAC} learning with nasty noise}.
\newblock {\em Theoretical Computer Science}, 288(2):255--275, 2002.

\bibitem[BFR14]{biggio2014security}
Battista Biggio, Giorgio Fumera, and Fabio Roli.
\newblock Security evaluation of pattern classifiers under attack.
\newblock {\em IEEE transactions on knowledge and data engineering},
  26(4):984--996, 2014.

\bibitem[BNL12]{biggio2012poisoning}
Battista Biggio, Blaine Nelson, and Pavel Laskov.
\newblock Poisoning attacks against support vector machines.
\newblock In {\em Proceedings of the 29th International Coference on
  International Conference on Machine Learning}, pages 1467--1474. Omnipress,
  2012.

\bibitem[BNS{\etalchar{+}}06]{barreno2006can}
Marco Barreno, Blaine Nelson, Russell Sears, Anthony~D Joseph, and J~Doug
  Tygar.
\newblock Can machine learning be secure?
\newblock In {\em Proceedings of the 2006 ACM Symposium on Information,
  computer and communications security}, pages 16--25. ACM, 2006.

\bibitem[Bor75]{borell1975brunn}
Christer Borell.
\newblock The brunn-minkowski inequality in gauss space.
\newblock {\em Inventiones mathematicae}, 30(2):207--216, 1975.

\bibitem[BPR18]{bubeck2018adversarial}
S{\'e}bastien Bubeck, Eric Price, and Ilya Razenshteyn.
\newblock {Adversarial examples from computational constraints}.
\newblock {\em arXiv preprint arXiv:1805.10204}, 2018.

\bibitem[CSV17]{charikar2017learning}
Moses Charikar, Jacob Steinhardt, and Gregory Valiant.
\newblock Learning from untrusted data.
\newblock In {\em Proceedings of the 49th Annual ACM SIGACT Symposium on Theory
  of Computing}, pages 47--60. ACM, 2017.

\bibitem[CW17]{CarliniWagner}
Nicholas Carlini and David~A. Wagner.
\newblock {Towards Evaluating the Robustness of Neural Networks}.
\newblock In {\em 2017 {IEEE} Symposium on Security and Privacy, {SP} 2017, San
  Jose, CA, USA, May 22-26, 2017}, pages 39--57, 2017.

\bibitem[DKK{\etalchar{+}}16]{diakonikolas2016robust}
Ilias Diakonikolas, Gautam Kamath, Daniel~M Kane, Jerry Li, Ankur Moitra, and
  Alistair Stewart.
\newblock Robust estimators in high dimensions without the computational
  intractability.
\newblock In {\em Foundations of Computer Science (FOCS), 2016 IEEE 57th Annual
  Symposium on}, pages 655--664. IEEE, 2016.

\bibitem[DKK{\etalchar{+}}17]{diakonikolas2017being}
Ilias Diakonikolas, Gautam Kamath, Daniel~M Kane, Jerry Li, Ankur Moitra, and
  Alistair Stewart.
\newblock Being robust (in high dimensions) can be practical.
\newblock In {\em International Conference on Machine Learning}, pages
  999--1008, 2017.

\bibitem[DKK{\etalchar{+}}18a]{diakonikolas2018robustly}
Ilias Diakonikolas, Gautam Kamath, Daniel~M Kane, Jerry Li, Ankur Moitra, and
  Alistair Stewart.
\newblock Robustly learning a gaussian: Getting optimal error, efficiently.
\newblock In {\em Proceedings of the Twenty-Ninth Annual ACM-SIAM Symposium on
  Discrete Algorithms}, pages 2683--2702. Society for Industrial and Applied
  Mathematics, 2018.

\bibitem[DKK{\etalchar{+}}18b]{diakonikolas2018sever}
Ilias Diakonikolas, Gautam Kamath, Daniel~M Kane, Jerry Li, Jacob Steinhardt,
  and Alistair Stewart.
\newblock Sever: A robust meta-algorithm for stochastic optimization.
\newblock {\em arXiv preprint arXiv:1803.02815}, 2018.

\bibitem[DKS17]{diakonikolas2017statistical}
Ilias Diakonikolas, Daniel~M Kane, and Alistair Stewart.
\newblock Statistical query lower bounds for robust estimation of
  high-dimensional {G}aussians and {G}aussian mixtures.
\newblock In {\em Foundations of Computer Science (FOCS), 2017 IEEE 58th Annual
  Symposium on}, pages 73--84. IEEE, 2017.

\bibitem[DKS18a]{diakonikolas2018list}
Ilias Diakonikolas, Daniel~M Kane, and Alistair Stewart.
\newblock List-decodable robust mean estimation and learning mixtures of
  spherical {G}aussians.
\newblock In {\em Proceedings of the 50th Annual ACM SIGACT Symposium on Theory
  of Computing}, pages 1047--1060. ACM, 2018.

\bibitem[DKS18b]{diakonikolas2018efficient}
Ilias Diakonikolas, Weihao Kong, and Alistair Stewart.
\newblock Efficient algorithms and lower bounds for robust linear regression.
\newblock {\em arXiv preprint arXiv:1806.00040}, 2018.

\bibitem[DMM18]{Adversarial:NIPS}
Dimitrios~I. Diochnos, Saeed Mahloujifar, and Mohammad Mahmoody.
\newblock {Adversarial Risk and Robustness: General Definitions and
  Implications for the Uniform Distribution}.
\newblock In {\em Conference on Neural Information Processing Systems (NIPS)},
  2018.

\bibitem[FFF18]{fawzi2018adversarial}
Alhussein Fawzi, Hamza Fawzi, and Omar Fawzi.
\newblock {Adversarial vulnerability for any classifier}.
\newblock {\em arXiv preprint arXiv:1802.08686}, 2018.

\bibitem[FMS15]{feige2015learning}
Uriel Feige, Yishay Mansour, and Robert Schapire.
\newblock Learning and inference in the presence of corrupted inputs.
\newblock In {\em Conference on Learning Theory}, pages 637--657, 2015.

\bibitem[FMS18]{feige2018robust}
Uriel Feige, Yishay Mansour, and Robert~E Schapire.
\newblock Robust inference for multiclass classification.
\newblock In {\em Algorithmic Learning Theory}, pages 368--386, 2018.

\bibitem[GM01]{giannopoulos2001euclidean}
Apostolos~A Giannopoulos and Vitali~D Milman.
\newblock {Euclidean structure in finite dimensional normed spaces}.
\newblock {\em Handbook of the geometry of Banach spaces}, 1:707--779, 2001.

\bibitem[GMF{\etalchar{+}}18]{gilmer2018adversarial}
Justin Gilmer, Luke Metz, Fartash Faghri, Samuel~S Schoenholz, Maithra Raghu,
  Martin Wattenberg, and Ian Goodfellow.
\newblock {Adversarial Spheres}.
\newblock {\em arXiv preprint arXiv:1801.02774}, 2018.

\bibitem[GMP18]{GoodfellowEtAl:MakeMLRobust}
Ian~J. Goodfellow, Patrick~D. McDaniel, and Nicolas Papernot.
\newblock {Making machine learning robust against adversarial inputs}.
\newblock {\em Communications of the {ACM}}, 61(7):56--66, 2018.

\bibitem[GSS15]{Adversarial::Harnessing}
Ian Goodfellow, Jonathon Shlens, and Christian Szegedy.
\newblock {Explaining and Harnessing Adversarial Examples}.
\newblock In {\em ICLR}, 2015.

\bibitem[Har66]{harper1966optimal}
Lawrence~H Harper.
\newblock Optimal numberings and isoperimetric problems on graphs.
\newblock {\em Journal of Combinatorial Theory}, 1(3):385--393, 1966.

\bibitem[KL93]{KearnsLi::Malicious}
Michael~J. Kearns and Ming Li.
\newblock {Learning in the Presence of Malicious Errors}.
\newblock {\em {SIAM} Journal on Computing}, 22(4):807--837, 1993.

\bibitem[KL17]{koh2017understanding}
Pang~Wei Koh and Percy Liang.
\newblock Understanding black-box predictions via influence functions.
\newblock In {\em International Conference on Machine Learning}, pages
  1885--1894, 2017.

\bibitem[Led01]{ledoux2001concentration}
Michel Ledoux.
\newblock {\em {The Concentration of Measure Phenomenon}}.
\newblock Number~89 in Mathematical Surveys and Monographs. American
  Mathematical Society, 2001.

\bibitem[L{\'e}v51]{levy1951problemes}
Paul L{\'e}vy.
\newblock {\em Probl{\`e}mes concrets d'analyse fonctionnelle}, volume~6.
\newblock Gauthier-Villars Paris, 1951.

\bibitem[LRV16]{lai2016agnostic}
Kevin~A Lai, Anup~B Rao, and Santosh Vempala.
\newblock Agnostic estimation of mean and covariance.
\newblock In {\em Foundations of Computer Science (FOCS), 2016 IEEE 57th Annual
  Symposium on}, pages 665--674. IEEE, 2016.

\bibitem[MDM18]{Mahloujifar2018:ALT}
Saeed Mahloujifar, Dimitrios~I Diochnos, and Mohammad Mahmoody.
\newblock {Learning under $p$-Tampering Attacks}.
\newblock In {\em ALT}, pages 572--596, 2018.

\bibitem[MM17]{pTampTCC17}
Saeed Mahloujifar and Mohammad Mahmoody.
\newblock {Blockwise p-Tampering Attacks on Cryptographic Primitives,
  Extractors, and Learners}.
\newblock In {\em Theory of Cryptography Conference}, pages 245--279. Springer,
  2017.

\bibitem[MMS{\etalchar{+}}18]{madry2017towards}
Aleksander Madry, Aleksandar Makelov, Ludwig Schmidt, Dimitris Tsipras, and
  Adrian Vladu.
\newblock {Towards Deep Learning Models Resistant to Adversarial Attacks}.
\newblock In {\em ICLR}, 2018.

\bibitem[MRT15]{mansour2015robust}
Yishay Mansour, Aviad Rubinstein, and Moshe Tennenholtz.
\newblock Robust probabilistic inference.
\newblock In {\em Proceedings of the twenty-sixth annual ACM-SIAM symposium on
  Discrete algorithms}, pages 449--460. Society for Industrial and Applied
  Mathematics, 2015.

\bibitem[MS86]{milman1986asymptotic}
Vitali~D Milman and Gideon Schechtman.
\newblock {\em Asymptotic theory of finite dimensional normed spaces}, volume
  1200.
\newblock Springer Verlag, 1986.

\bibitem[PMSW16]{papernot2016towards}
Nicolas Papernot, Patrick McDaniel, Arunesh Sinha, and Michael Wellman.
\newblock Towards the science of security and privacy in machine learning.
\newblock {\em arXiv preprint arXiv:1611.03814}, 2016.

\bibitem[PMW{\etalchar{+}}16]{Defenses:Distillation}
Nicolas Papernot, Patrick~D. McDaniel, Xi~Wu, Somesh Jha, and Ananthram Swami.
\newblock {Distillation as a Defense to Adversarial Perturbations Against Deep
  Neural Networks}.
\newblock In {\em {IEEE} Symposium on Security and Privacy, {SP} 2016, San
  Jose, CA, USA, May 22-26, 2016}, pages 582--597, 2016.

\bibitem[PSBR18]{prasad2018robust}
Adarsh Prasad, Arun~Sai Suggala, Sivaraman Balakrishnan, and Pradeep Ravikumar.
\newblock Robust estimation via robust gradient estimation.
\newblock {\em arXiv preprint arXiv:1802.06485}, 2018.

\bibitem[RNH{\etalchar{+}}09]{rubinstein2009antidote}
Benjamin~I.P. Rubinstein, Blaine Nelson, Ling Huang, Anthony~D. Joseph,
  Shing-hon Lau, Satish Rao, Nina Taft, and J.D. Tygar.
\newblock Antidote: understanding and defending against poisoning of anomaly
  detectors.
\newblock In {\em Proceedings of the 9th ACM SIGCOMM conference on Internet
  measurement conference}, pages 1--14. ACM, 2009.

\bibitem[SST{\etalchar{+}}18]{schmidt2018adversarially}
Ludwig Schmidt, Shibani Santurkar, Dimitris Tsipras, Kunal Talwar, and
  Aleksander Madry.
\newblock {Adversarially Robust Generalization Requires More Data}.
\newblock {\em arXiv preprint arXiv:1804.11285}, 2018.

\bibitem[ST78]{sudakov1978extremal}
Vladimir~N Sudakov and Boris~S Tsirel'son.
\newblock Extremal properties of half-spaces for spherically invariant
  measures.
\newblock {\em Journal of Soviet Mathematics}, 9(1):9--18, 1978.

\bibitem[STS16]{shen2016uror}
Shiqi Shen, Shruti Tople, and Prateek Saxena.
\newblock A uror: defending against poisoning attacks in collaborative deep
  learning systems.
\newblock In {\em Proceedings of the 32nd Annual Conference on Computer
  Security Applications}, pages 508--519. ACM, 2016.

\bibitem[SZS{\etalchar{+}}14]{Szegedy:intriguing}
Christian Szegedy, Wojciech Zaremba, Ilya Sutskever, Joan Bruna, Dumitru Erhan,
  Ian Goodfellow, and Rob Fergus.
\newblock {Intriguing properties of neural networks}.
\newblock In {\em ICLR}, 2014.

\bibitem[Tal95]{talagrand1995concentration}
Michel Talagrand.
\newblock Concentration of measure and isoperimetric inequalities in product
  spaces.
\newblock {\em Publications Math{\'e}matiques de l'Institut des Hautes Etudes
  Scientifiques}, 81(1):73--205, 1995.

\bibitem[Val85]{Valiant::DisjunctionsConjunctions}
Leslie~G. Valiant.
\newblock {Learning disjunctions of conjunctions}.
\newblock In {\em IJCAI}, pages 560--566, 1985.

\bibitem[WC18]{PoisoningOnlineWangChaudhuri}
Yizhen Wang and Kamalika Chaudhuri.
\newblock {Data Poisoning Attacks against Online Learning}.
\newblock {\em arXiv preprint arXiv:1808.08994}, 2018.

\bibitem[XBB{\etalchar{+}}15]{xiao2015feature}
Huang Xiao, Battista Biggio, Gavin Brown, Giorgio Fumera, Claudia Eckert, and
  Fabio Roli.
\newblock Is feature selection secure against training data poisoning?
\newblock In {\em ICML}, pages 1689--1698, 2015.

\bibitem[XEQ18]{Adversarial::FeatureSqueezing}
Weilin Xu, David Evans, and Yanjun Qi.
\newblock {Feature Squeezing: Detecting Adversarial Examples in Deep Neural
  Networks}.
\newblock In {\em NDSS}, 2018.

\end{thebibliography}
\iffollowingorders
\else
\appendix
\section{Risk and Robustness Based on Hypothesis's Prediction Change} \label{sec:PC}

The work of Szegedy et al.~\cite{Szegedy:intriguing}, as well as a big portion of subsequent work on adversarial examples, relies on defining adversarial risk and robustness of a hypothesis $h$ based on the amount of adversarial perturbations that change the prediction of $h$. Their definition is independent of the concept function  $c$ determining the ground truth. In particular, for a given example $(x, c(x))$ where the prediction of the hypothesis is $h(x)$ (that might indeed be different from $c(x)$), an adversarial perturbation of $x$ is $r$ such that for the instance $x' = x + r$ we have $h(x') \neq h(x)$ (where $h(x')$ may or may not be equal to $c(x')$). Hence, since the attacker only cares about changing the prediction of the hypothesis $h$, we refer to adversarial properties (be it adversarial perturbations, adversarial risk, adversarial robustness) under this definition  as adversarial properties based on ``prediction change'' (PC for short)-- as opposed to adversarial properties based on the ``error region''   in Definition \ref{def:ER}.

In this section, we show that using the concentration function $\con(\cdot)$ and our proofs of Section \ref{sec:ER}, one can also bound the PC risk and robustness of hypotheses assuming that we have a concentration function. Then, one can use any concentration function (e.g., those of \Levy families) and  obtain the desired upper/lower bounds, just as how we did so for the the results of Subsection \ref{sec:Levy}.
\remove{
The work of Fawzi et al.~\cite{fawzi2018adversarial} builds upon the notion of PC adversarial perturbations. In their work, they  show that when the underlying distribution is Gaussian, any classifier is vulnerable to adversarial perturbations when one uses the PC risk and robustness definitions.
On the other hand, the work of Diochnos et al.~\cite{Adversarial:NIPS} and  is using ``error region'' adversarial perturbations as in the current paper. Similarly to the work of Fawzi et al., the work of Diochnos et al.~uses an isoperimetric inequality in order to show vulnerability of classifiers under a specific distribution, however the setting is different; namely the underlying distribution is uniform over the Boolean hypercube.

The work and the results of our current paper are inspired by the above mentioned work of Fawzi et al.~and Diochnos et al.~and use isoperimetric inequalities in order to show vulnerability to adversarial examples (based on the ``error region'') for L{\'e}vy families, where an important special case is product distributions over the Boolean hypercube. However, our techniques and results also extend to adversarial definitions based on prediction change and in what follows we exemplify precisely this.
}

\paragraph{Focusing on the hypothesis class.} Whenever we consider a classification problem $(\X,\Y,\msr,\hypoC,\metric)$ without explicitly denoting the concept class $\C$, we mean that $(\X,\Y,\msr,\C,\hypoC,\metric)$ is nice for the trivial set $\C$ of constant functions that output either of $y \in Y$. The reason for this definition is that basically, below we will require some concept class, and all we want is that preimages of specific labels under any $h$ are measurable sets, which is implied if the problem is nice with the simple $\C$ described.

\begin{definition}[Prediction-change adversarial risk and robustness] \label{def:PC}
Let $(\X,\Y,\msr,\hypoC,\metric)$ be a nice classification problem. For $h \in \H$, and $\ell \in \Y$, we define $h^\ell = \set{x \in \X \mid h(x) = \ell}$. Then, for any $h \in H$, we define the following.
\begin{itemize}
    \item {\bf Prediction change (PC) risk.}  The   \emph{PC risk} under $b$-perturbation  is 
$$\Risk\PC_b(h) = \Pr_{x \gets \msr}\left[\Exists x' \in \Ball_b(x) , h(x) \neq h(x') \right].$$
    \item {\bf Target-label PC risk.}  For  $\ell \in \Y$ and $b \in \Rplus$, the             \emph{$\ell$-label (PC) risk} under $b$-perturbation  is 
$$\Risk^\ell_b(h) = \Pr_{x \gets \msr}\left[\Exists x' \in \Ball_b(x) \cap  h^\ell \right] = \msr(h^\ell_b).$$
    \item {\bf PC robustness.} For a given non-constant $h \in \H$,  we define the {\em PC robustness} as the expected perturbation needed to change the labels as follows
\begin{align*}
\Rob\PC(h)
&=\expectedsub{\substack{x \gets \msr \\ \ell = h(x)}}{\metric(x,\X \sm h^\ell)}.
\end{align*}
    \item {\bf Target-label PC robustness.} For  $\ell \in \Y$ and a given non-constant $h \in \H$,  we define the {\em $\ell$-label (PC) robustness} as the expected perturbation needed to make the label always $\ell$ defined as
    \begin{align*}
\Rob^\ell(h)
&=\expectedsub{\substack{x \gets \msr}}{\metric(x,h^\ell)}.
\end{align*}
\end{itemize}
\end{definition}

\begin{theorem}[PC risk and robustness in concentrated spaces] \label{thm:PC}
Let $(\X,\Y,\msr,\hypoC,\metric)$ be a nice classification problem. For any $h\in \hypoC$ that is \emph{not} a constant function, the following hold.
\begin{enumerate}
    \item Let $\eps \in [0,1/2]$ be such that  $\msr(h^\ell) \leq 1-\eps$ for all $\ell \in \Y$. If $\con(b_1) < \eps/2 $ and $\con(b_2) \leq \gamma/2$, then for $b=b_1+b_2$ we have
    $$\Risk_{b}\PC(h)\geq 1-\gamma.$$
    \item If $\con(b_1) < \msr(h^\ell)$ and  $\con(b_2) \leq \gamma$ then for $b=b_1+b_2$ we have
    $$\Risk^\ell_b(h) \geq 1-\gamma.$$
    \item If $\Risk\PC_b(h)\geq \frac{1}{2}$, then
    $$\Rob\PC(h)\leq b + \int_0^\infty \con(z)\cdot dz.$$
    \item If $\con(b) < \msr(h^\ell)$, then
    $$\Rob^\ell(h)\leq b + \int_0^\infty \con(z)\cdot dz.$$
\end{enumerate}
\end{theorem}
\begin{proof} We prove the parts in order.

\begin{enumerate}
    \item Let $b=b_1+b_2$. Also, for a set $\cZ \se \Y$, let  $h^{\cZ} = \cup_{\ell \in \cZ} h^\ell$.
Because for all $\ell \in \Y$ we have $\msr(h^\ell)\leq 1-\eps$, it can be shown that there is a set $\Y^1 \subset \Y$ such that $\msr(h^{\Y^1}) \in ({\eps}/{2}, 1/2]$. Let $\X^1 = \set{x\in\X \mid h(x) \in \Y^1}$ and $\X^2= \X\sm\X^1$. We know that $\msr(\X^1) > {\eps}/{2}$, so
$$\msr(\X^1_b) \geq 1-\gamma/2.$$
On the other hand, we know that $\msr(\X^2)\geq 1/2$, therefore we have
$$\msr(\X^2_b)\geq \msr(\X^2_{b_2})\geq 1-\gamma/2.$$
By a union bound we conclude that
$$\msr(\X^1_b\cap \X^2_b) \geq 1-\gamma$$
which implies that $\Risk\PC_b(h)\geq1-\gamma$. The reason is that for any $x\in \X^1_b\cap \X^2_b$ there are $x^1,x^2\in \Ball(x,b)$ such that $h(x^1)\in \Y^1$ and $h(x^2)\in \Y \sm \Y^1$ which means either $h(x)\neq h(x^1)$ or $h(x) \neq h(x^2)$.
    \item The proof Part 2 directly follows from the definition of $\con$ and an argument identical to that of Part 2 of Theorem \ref{thm:ConctToRisk}.
    \item  
Let $\cE = \set{x\in \X \mid \Exists x' \in \Ball_b(x) , h(x) \neq h(x')}$. We know that $\msr(\cE)\geq 1/2$, therefore by Theorem \ref{thm:ConcToRob} we have
$$\Rob(\cE) \leq \int_0^\infty \con(z)\cdot dz.$$
On the other hand, for every $x\in \X$ where $\ell = h(x)$, we have
$\metric(x,\X \sm h^\ell) \leq b + \metric(x,\cE)$ because we know that for any $x'\in \cE$ there exist some $x'' \in \Ball(x',b)$ such that $h(x')\neq h(x'')$. Therefore, we get that either $h(x')\neq h(x)$ or $h(x'')\neq h(x)$, which implies $\metric(x,\X \sm h^{h(x)}) \leq b + \metric(x,\cE)$. Thus, we have
$$\Rob\PC(h) \leq b + \Rob(\cE) \leq b + \int_0^\infty \con(z)\cdot dz.$$

    \item Part 4 follows from an argument that is identical to that of Theorem \ref{thm:ConcToRob}.\qedhere
\end{enumerate}
\end{proof}

the following corollary directly follows Theorem \ref{thm:PC} above and Definition \ref{def:Levy} of \Levy families, just the same way Corollary \ref{cor:RiskRobLevy} could be derived from  Theorems \ref{thm:ConctToRisk} and \ref{thm:ConcToRob} (by going through a variant of Theorems \ref{thm:RiskRobLevy} for PC risk and robustness that we skip) to get asymptotic bounds of risk and robustness of classification tasks over \Levy spaces.

\begin{corollary}[Asymptotic PC risk and robustness in normal \Levy families] \label{cor:PCRiskRobLevy}
Let $\mathsf{P}_n$ be a nice classification problem defined over a metric probability space that is a normal \Levy family.
\begin{enumerate}
    \item  {\bf PC risk and robustness.} If for all $\ell \in \Y$ it holds that $\msr(h^\ell) \leq 0.99$ (i.e., $h$ is not almost constant), then the amount of perturbations needed to achieve
  PC risk $0.99$ is $O(1/\sqrt n)$ and  the (full) PC robustness of $h$ is also $O(1/\sqrt n)$.
    \item {\bf Target-label PC risk and robustness.} If a particular label $\ell$ happens with constant probability $\msr(h^\ell) = \Omega(1)$, then the perturbation needed to increase $\ell$-label PC risk to $0.99$ and the $\ell$-label PC robustness of $h$ are both at most $O(1/\sqrt n)$. Furthermore,  if $\msr(h^\ell) \geq \exp (-o(n))$ is  subexponentially large, then the perturbation needed to increase the $\ell$-label PC risk to $1-\exp(-o(n))$  and the $\ell$-label PC robustness of $h$ are at both most $o(1)$.
\end{enumerate}
\end{corollary}



\fi

\end{document}